\def\year{2022}\relax
\documentclass[letterpaper]{article} % DO NOT CHANGE THIS
\usepackage{aaai22}  % DO NOT CHANGE THIS
\usepackage{times}  % DO NOT CHANGE THIS
\usepackage{helvet}  % DO NOT CHANGE THIS
\usepackage{courier}  % DO NOT CHANGE THIS
\usepackage[hyphens]{url}  % DO NOT CHANGE THIS
\usepackage{graphicx} % DO NOT CHANGE THIS
\usepackage{natbib}  % DO NOT CHANGE THIS AND DO NOT ADD ANY OPTIONS TO IT
\usepackage{caption} % DO NOT CHANGE THIS AND DO NOT ADD ANY OPTIONS TO IT
\DeclareCaptionStyle{ruled}{labelfont=normalfont,labelsep=colon,strut=off} % DO NOT CHANGE THIS
\usepackage{algorithm}
\usepackage{algorithmic}

\usepackage{graphicx}
\usepackage{enumerate}
\usepackage{subfigure}
\usepackage{multirow}
\floatname{algorithm}{Algorithm}  
\usepackage{amsmath,amssymb,amsthm}
\DeclareMathOperator*{\prog}{prog}
\DeclareMathOperator*{\argmax}{argmax}
\DeclareMathOperator*{\argmin}{argmin}
\renewcommand{\P}{\mathcal{P}}
\renewcommand{\iff}{\Leftrightarrow}
\newcommand{\A}{\mathcal{R}}
\newcommand{\until}{\mathsf{U}}
\newcommand{\true}{\top}
\newcommand{\false}{\bot}

\newtheorem{theorem}{Theorem}
\newtheorem{define}{Definition}

\usepackage{newfloat}
\usepackage{listings}
\floatstyle{ruled}
\newfloat{listing}{tb}{lst}{}
\floatname{listing}{Listing}
\title{Lifelong Reinforcement Learning with Temporal Logic Formulas\\ and Reward Machines}
\title{Lifelong Reinforcement Learning with Temporal Logic Formulas\\ and Reward Machines}
\author {
    Xuejing Zheng\textsuperscript{\rm 1},
    Chao Yu\textsuperscript{\rm 1*},
    Chen Chen\textsuperscript{\rm 2},
    Jianye Hao\textsuperscript{\rm 2},
    Hankz Hankui Zhuo\textsuperscript{\rm 1}
}
\affiliations {
    \textsuperscript{\rm 1} Dept. of Computer Science, Sun Yat-sen University\\
    \textsuperscript{\rm 2} Huawei Noah's Ark Lab\\
    zhengxj28@mail2.sysu.edu.cn, yuchao3@mail.sysu.edu.cn, chenchen9@huawei.com, haojianye@huawei.com, zhuohank@mail.sysu.edu.cn
}

\usepackage{bibentry}
\begin{document}

\maketitle

\begin{abstract}
% Temporal logic formulas have been leveraged to specify tasks with high-level events from the environment, while \textit{Reward Machines} (RM)--a type of finite state machines--can be used to express the structural reward functions of tasks for more efficient learning. 
%  use them for lifelong RL, core idea of LSRM
Continuously learning new tasks using high-level ideas or knowledge is a key capability of humans. In this paper, we propose \textit{Lifelong reinforcement learning with Sequential linear temporal logic formulas and Reward Machines} (LSRM), which enables an agent to leverage previously learned knowledge to fasten learning of logically specified tasks. For the sake of more flexible specification of tasks, we first introduce \textit{Sequential Linear Temporal Logic} (SLTL), which is a supplement to the existing \textit{Linear Temporal Logic} (LTL) formal language. We then utilize \textit{Reward Machines} (RM) to exploit structural reward functions for tasks encoded with high-level events, and propose automatic extension of RM and efficient knowledge transfer over tasks for continuous learning in lifetime.
%Building upon the RM, various Q-value composition methods are then proposed to transfer knowledge from the learned Q-functions to the target tasks. 
Experimental results show that LSRM outperforms the methods that learn the target tasks from scratch by taking advantage of the task decomposition using SLTL and knowledge transfer over RM during the lifelong learning process.
% the temporal logic formulas xxxxx, so that the agent can exploit task modularity by decomposing formulas to transfer knowledge from previous learned tasks over its lifetime.

% We then propose \textit{Lifelong learning with SLTL formulas and RM} (LSRM) by maintaining a memory over lifetime. The memory includes an RM, a set of previous learned tasks and a set of Q-functions. 
% When faced with new tasks, the memory is updated by automatically extending the reward machine and transferring Q-functions of learned tasks to new tasks. Evaluation in two benchmark shows that our learning approach improves the sample efficiency by a large margin compared with direct learning methods.
\end{abstract}

\section{Introduction}
% background

%Lifelong machine learning (Thrun 1996) examines learning multiple tasks in sequence, with an emphasis on how previous knowledge of different tasks can be used to improve the training time and learning of current tasks. The hope is that, if sequential learning can be repeated indefinitely, then a system can continue to learn over the course of its lifetime. A continual learning agent is thought to be an important step toward general artificial intelligence.

There are at least two significant abilities of human intelligence: (i) storing learned skills in memory over lifetime and leveraging them when encountering new tasks; and (ii) utilizing high-level ideas or knowledge for more efficient reasoning and learning. These abilities enable humans to adapt quickly in environments where tasks and experiences change over time. 
% However, traditional \textit{Reinforcement Learning} (RL) \cite{sutton2018reinforcement} lacks such abilities, which makes it inefficient to learn from scratch without high-level knowledge.
% prior work (motivation)
%effect
\textit{Lifelong Reinforcement Learning} (LRL) \cite{abel2018policy,brunskill2014pac,garcia2019meta} formalizes the problem of building taskable agents by exploiting knowledge gained in previous tasks to improve performance in new but related tasks. Solving the LRL problem is an essential step toward general artificial intelligence as it allows agents to continuously adapt to changes in the environment with minimal human intervention, which is a key feature of human learning.

There has recently been a surge of interest in methods for achieving efficient LRL, utilizing techniques such as network consolidation \cite{schwarz2018progress} or freezing \cite{rusu2016progressive}, rehearsal via experience replay \cite{isele2018selective,rolnick2018experience}, and value-function/policy initialization \cite{abel2018policy}. Remarkably, the line of these works has mainly focused on the continuous learning settings, where a series of related tasks are drawn from a task distribution. Attention is restricted to subclasses of MDPs by making structural assumptions about which MDP components (rewards or transition probabilities) may change in support of the generation of tasks. While this kind of assumptions is reasonable in most real life situations, there are also  scenarios when either task specification is non-Markovian and thus difficult to be expressed analytically as a reward function, or the sequential tasks cannot be generated from an underlying distribution when expressed logically using formal languages \cite{linz2006introduction,pnueli1977temporal}. For instance, consider a scenario when an agent has learned the task of ``\emph{delivering coffee and mail to office}''. When facing a new task of ``\emph{delivering coffee or mail to office}'', it is unclear how existing LRL methods would model these tasks as the same distribution of MDP and enable efficient transfer learning among these tasks. This limitation contradicts the human ability of compositional learning using high-level ideas or knowledge, \emph{i.e.,} understanding novel situations by combining and reasoning over already known primitives.

%In order to specify these tasks encoded with high-level events of the environment, \textit{Linear Temporal Logic} (LTL) is utilized, by defining the successful and unsuccessful executions of a task\cite{li2017reinforcement,toro2018teaching}. \textit{Reward Machines} (RM) \cite{icarte2018using}--a type of finite states machine--are then proposed to express the structural reward functions of tasks, along with various learning algorithms such as \textit{Q-learning for Reward Machines} (QRM) to learn such tasks. QRM not only converges to an optimal policy in the tabular case, but also outperforms Hierarchical Reinforcement Learning (HRL) methods \cite{icarte2018using,toro2020reward}.
% , thereby being widely used in RL \cite{camacho2020disentangled,neary2020reward,icarte2019learning}.  

%Compared to existing skill composition methods, we are able to learn and compose logically complex tasks that would otherwise be difficult to analytically expressed as a reward function.

In this paper, we investigate LRL problems when the series of tasks do not necessarily share the same MDP structure, but instead are specified with high-level events using \textit{Linear Temporal Logic} (LTL) \cite{li2017reinforcement,toro2018teaching}. The basic intuition is to exploit task modularity and decomposition with higher abstraction and succinctness to facilitate transfer learning in target tasks. We first introduce the \textit{Sequential Linear Temporal Logic} (SLTL), which is a supplement to LTL by adding a new operator \textit{``then''}, and also prove the operator laws of SLTL to provide more flexible and rich specification of a task. In order to enable more efficient task learning, \textit{Reward Machines} (RM) \cite{icarte2018using} are utilized to express structural reward functions for tasks encoded with high-level events. Synthesizing the merits of task modularity using SLTL and policy learning over RM, we propose the \textit{Lifelong reinforcement learning with SLTL and RM} (LSRM) method, which stores and leverages high-level knowledge in a memory for more efficient lifelong learning of logically specified tasks. The memory contains an RM, which can be automatically updated when facing a set of target tasks, and the high-level knowledge stored in the memory can be transferred to a new decomposed target task using a number of value composition methods. 
% While it is possible to express task encoded with high-level events by LTL formulas and utilize its corresponding RM to facilitate learning, it is not straightforward to leverage previous learned tasks to construct a new RM and learn with it when faced with new tasks over lifetime.
%Building upon the RM in memory, we then propose various value composition methods to transfer the high-level knowledge from the memory to a new task.
%including \textit{Average Composition, Max Composition, Left Composition} and \textit{Right Composition}. We evaluate these methods in two benchmark domains. Experimental results show that  when the target tasks are composed by the \textit{``and'', ``or''} and \textit{``then''} operator, the \textit{Average Composition, Max Composition} and \textit{Left Composition} result in the best performance, respectively. 
We evaluate the performance of LSRM in the \textsc{OfficeWorld} and \textsc{MineCraft} domains. Experiments show that LSRM enables the agent to learn target tasks more efficiently compared to the direct methods without transfer learning.

The remaining part of the paper is organized as follows. Section 2 provides a background introduction. Section 3 introduces the SLTL language. Section 4 presents the LSRM and  Section 5 provides experimental studies. Section 6 reviews some related works, and finally, Section 7 concludes with some directions of future work.

\section{Preliminaries}
We provide a preliminary introduction to RL, LTL and RM in this section. Please refer to \cite{sutton2018reinforcement,pnueli1977temporal,toro2020reward} for more details in these topics.
\subsection{RL}
The RL problem consists of an agent interacting with an unknown environment~\cite{sutton2018reinforcement}, which can be modeled as a \textit{Markov Decision Process} (MDP) by a tuple $M=(S,A,r,p,\gamma)$, where $S$ is a finite set of states, $A$ is a finite set of actions, $r:S\times A\times S\to \mathbb{R}$ is a reward function, $p:S\times A\times S\to [0,1]$ is a probabilistic transition function, and $\gamma\in (0,1]$ is a discount factor. 

A policy is a mapping $\pi:S\times A\to [0,1]$, and $\pi(s,a)$ means the probability of choosing action $a$ in state $s$. The Q-function $Q_\pi(s,a)$ following the policy $\pi$ is the expected discounted reward of choosing action $a$ in state $s$ under policy $\pi$, \textit{i.e.,} $Q_\pi(s,a)=\mathbb{E}_\pi[\sum_{k=0}^\infty \gamma^k r_{t+k+1}\mid s_t=s,a_t=a]$,~where $s_{k+1}\sim p(s_k,a_k,\cdot), r_{k+1}= r(s_k,a_k,s_{k+1})$, and $a_k\sim \pi(s_k,\cdot)$. The goal of RL is to learn an optimal policy $\pi^*$, which maximizes the expected discounted reward for each $s\in S,a\in A$, \textit{i.e.,} $Q^*(s,a):=Q_{\pi^*}(s,a)=\max_\pi Q_\pi(s,a)$. A well-known approach for calculating the optimal Q-function in tabular case is Q-learning \cite{watkins1992q}. Its one-step updating rule is given by $Q(s,a)\xleftarrow{\alpha}r(s,a,s')+\gamma \max_{a'} Q(s',a')$, where $x\xleftarrow{\alpha}y$ means $x\leftarrow x+\alpha(y-x)$. The action $a$ is chosen by using certain exploration strategies, such as the $\epsilon$-greedy policy, \textit{i.e.,} choosing a random action with probability $\epsilon$, while choosing $\argmax_{a'}Q(s,a')$ with probability $1-\epsilon$.

\subsection{LTL}
LTL is a propositional modal logic with temporal modalities \cite{pnueli1977temporal}. It has been used to specify tasks encoded with high-level events in RL by characterizing the successful and unsuccessful executions \cite{toro2018teaching,li2017reinforcement}. Each high-level event is represented by a \textit{propositional variable}, and the set of all propositional variables is denoted by $\P$. The set of events that occur at time $t=i$ under state $s_i$ is a \textit{label}, denoted by $l_i\subseteq \P$, which is given by a \textit{labelling function} $L:S\to 2^\P, l_i=L(s_i)$. As an illustrative example, consider the \textsc{OfficeWorld} domain presented in Figure \ref{fig:OfficeWorld}. The propositional variables can be $\P=\{c,m,o,*,A,B,C,D\}$, where $c$ is \textit{``getting coffee''}, $m$ is \textit{``getting mail''}, $o$ is \textit{``at office''}, $*$ is \textit{``furniture''}, and $A,B,C,D$ is \textit{``at A, B, C, D''}, respectively. An event $p\in \P$ occurs if and only if the agent is located at the grid marked by $p$. Hence the labelling function is $L(s)=\{p\}$ if $s$ is marked by $p$, and $L(s)=\varnothing$ otherwise. Given the propositional variables $\P$, LTL formulas can be conducted from the standard Boolean operators $\land$ (\textit{and}), $\neg$ (\textit{negation}), and temporal operators $\bigcirc$ (\textit{next}), $\until$ (\textit{until}). Other operators can be derived from the operators above. For example, the operator $\lor$ (\textit{or}) and $\Diamond$ (\textit{eventually}) are defined as $\varphi\lor\psi=\neg(\neg\varphi\land\neg\psi)$ and $\Diamond \varphi=\true\until\varphi$, respectively, where $\true=\varphi\lor\neg\varphi$ is the formula \textit{``true''}. Formally, the syntax of an LTL formula is defined as $\varphi:=p\mid \neg \varphi \mid \varphi\land\psi\mid \bigcirc\varphi \mid \varphi\until\psi$, where $p\in \P$. 

The semantic of LTL is defined over an infinite sequence of labels, denoted $\lambda=l_0l_1l_2\cdots$. We use $\lambda\models \varphi$ to denote that an LTL formula $\varphi$ is determined to be true by a sequence $\lambda$, which is formally defined as follows: (i) $\lambda\models p\iff p\in l_0$, where $p\in \P$; (ii) $\lambda\models \neg \varphi\iff \lambda\not\models\varphi$; (iii) $\lambda\models\varphi_1\land\varphi_2\iff \lambda\models\varphi_1$ and $\lambda\models\varphi_2$;  (iv) $\lambda\models \bigcirc\varphi\iff \lambda^1\models \varphi$; and (v) $\lambda\models \varphi_1\until\varphi_2\iff \exists j>0,\lambda^j\models \varphi_2$ and $\forall i\leq j, \lambda^i\models \varphi_1$. The notation $\lambda^i$ denotes the postfix $l_il_{i+1}l_{i+2}\cdots$ of the sequence. Based on the above definitions, we can define tasks using LTL formulas. For instance, the task \textit{``eventually reaching office''} is defined as $\Diamond o$, and the task \textit{``delivering coffee to office''} is defined as $\Diamond (c\land\bigcirc\Diamond o)$.

In order to specify the remaining module of a task after proceeding a sequence of labels, LTL~\textit{progression}~\cite{bacchus2000using} has been proposed. For example, we can progress the formula $\Diamond (c\land\bigcirc\Diamond o)$ by a label $\{c\}$ and get a new formula $\Diamond o$, implying that the remaining task is \textit{``going to office''} after getting the coffee. Formally, an LTL progression maps an LTL formula $\varphi$ and a label $l$ to another LTL formula, denoted as $\prog(\varphi,l)$, which is defined recursively as follows: (i) $\prog(p,l)=\true$ if $p\in l$, where $p\in \P$; (ii) $\prog(\neg\varphi,l)=\neg\prog(\varphi,l)$; (iii) $\prog(\varphi_1\land\varphi_2,l)=\prog(\varphi_1,l)\land\prog(\varphi_2,l)$; (iv) $\prog(\bigcirc\varphi,l)=\varphi$; and (v) $\prog(\varphi_1\until\varphi_2,l)=\prog(\varphi_2,l)\lor(\prog(\varphi_1,l)\land\varphi_1\until\varphi_2)$. It has been theoretically proved that a sequence of labels satisfies an LTL formula if and only if the postfix of the sequence satisfies the progressed formulas, \textit{i.e.,} $\lambda^i\models\varphi\iff \lambda^{i+1}\models \prog(\varphi,l_i)$ \cite{bacchus2000using}.

\begin{figure}
    \centering
    \subfigure[\textsc{OfficeWorld}.]{\label{fig:OfficeWorld}
    \includegraphics[width=0.45\linewidth]{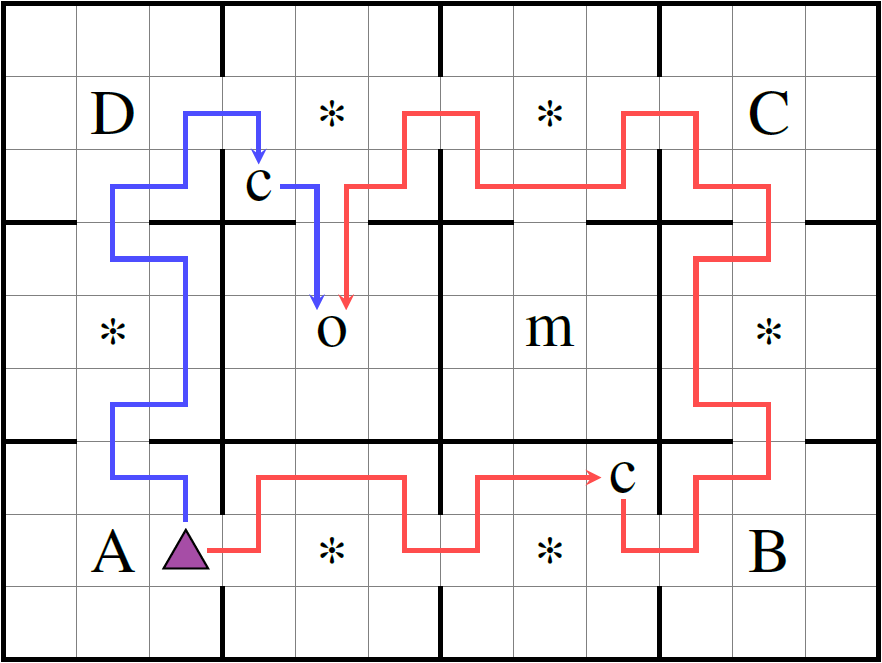}}
    \subfigure[An example of RM.]{\label{fig:example_rm1}
    \includegraphics[width=0.45\linewidth]{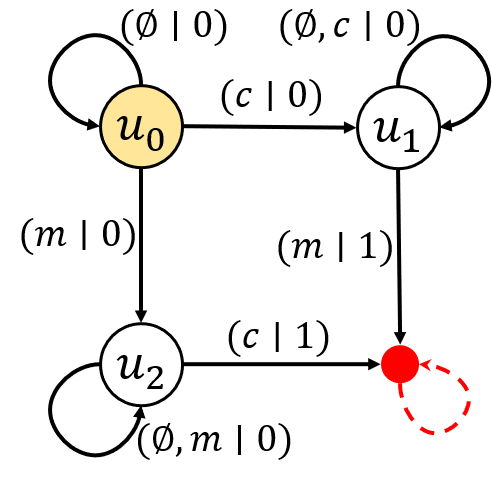}}
    \subfigure[Encoding RM with LTL.]{\label{fig:example_rm2}
    \includegraphics[width=0.45\linewidth]{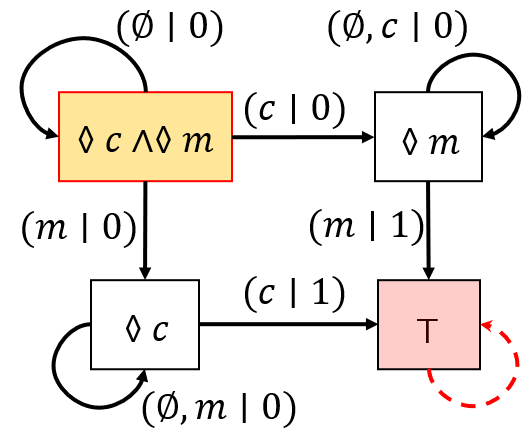}}
    \subfigure[Reward shaping for RM.]{\label{fig:example_rm3}
    \includegraphics[width=0.45\linewidth]{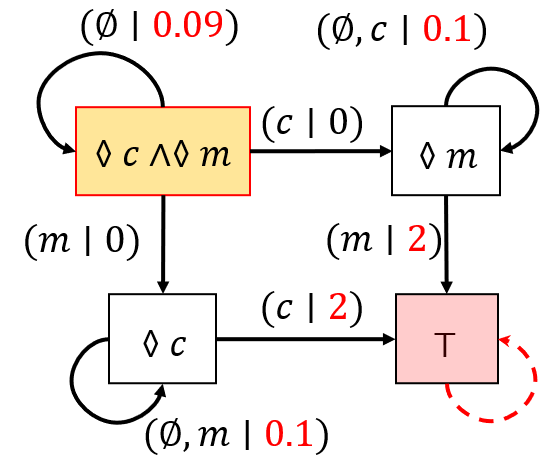}}
    \caption{Illustrations of the \textsc{OfficeWorld} environment and RM.}
\end{figure}

\subsection{RM}
RM can be used to reveal the structure of non-Markovian reward functions of tasks that are encoded with high-level events (\textit{i.e.,} propositional variables), formally defined as follows \cite{toro2020reward}.
\begin{define}
Given the propositional variables $\P$ and a set of all possible labels $\Sigma\subseteq 2^\P$, an RM is a tuple $\A=(U,u_0,F,\delta,R)$, where $U$ is a finite set of states, $u_0$ is an initial state, $F\subseteq U$ are terminal states, $\delta: U\times \Sigma\to U$ is the transition function,  and $R:U\times\Sigma\to\mathbb{R}$ is the reward function.
\end{define}
Distinguished from the states of the environment (denoted by $s$), the \textit{states} of RM (denoted by $u$) and the transitions among them are generally given by the prior knowledge of specific tasks. Figure \ref{fig:example_rm1} gives an example of RM that represents the task \textit{``eventually getting coffee and mail''} in \textsc{OfficeWorld}, where the nodes are states and the edges are transitions among states. Each transition is encoded with a tuple $(l_1,l_2,\cdots,l_n\mid r)$, where each $l_i$ is a label and $r$ is the output reward. The initial state is colored in yellow and the terminal state is colored in red. In order to  enrich the expressiveness of RM, each state of RM can also be encoded with an LTL formula (Figure \ref{fig:example_rm2}), implying a portion of a task that the agent has not completed yet.
%These encodings enable us to \textit{automatically} transfer knowledge by \textit{matching formulas} between the target tasks and previous learned tasks. 
Especially, the terminal states are encoded with $\true$(\textit{true}) or $\false$(\textit{false}), which indicates the completion or failure of a task, respectively. When an RM transits to terminal states, the current episode ends. The reward function of RM is then defined as follows:
\begin{equation}
\label{equ:reward function}
    R(\psi,l)=
    \begin{cases}
    1,\quad \text{if }\psi\neq \true, \delta(\psi,l)=\true;\\
    0,\quad \text{otherwise}.
    \end{cases}
\end{equation}
In order to induce denser rewards than binary rewards above, \textit{automatic reward shaping} for RM~\cite{camacho2019ltl,toro2020reward} is proposed. It modifies the reward function as $R'(\psi,l)=R(\psi,l)+\gamma \Phi(\delta(\psi,l))-\Phi(\psi)$, where $\Phi:U\to\mathbb{R}$ is the potential function calculated by value iteration. The modified reward function of the example above is presented in Figure \ref{fig:example_rm3}.

The QRM algorithm \cite{icarte2018using} is then proposed to leverage RM to learn tasks. QRM maintains a Q-function $Q^u$ for each RM state $u\in U$ (or formula), and updates all the Q-functions using \textit{internal} rewards and transitions of RM with one experience $(s,a,s')$. Formally, the Q-function of each state in RM is updated by
\begin{equation}
Q^u(s,a)\xleftarrow{\alpha} R(u,l)+\gamma \max_{a'}Q^{u'}(s',a'), \forall u\in U
\end{equation}
where $l=L(s)$ is the current label, $R(u,l)$ is the internal reward and $u'=\delta(u,l)$ is the internal next state of RM. QRM not only converges to an optimal policy in tabular cases, but also outperforms the Hierarchical RL (HRL) methods which might converge to suboptimal policies \cite{icarte2018using}.

% In our framework, we denote $F^+, F^-$ to be a set of \textit{positive} and \textit{negative} terminal states, implying the agent completes or fails a task, respectively, where $F^+,F^-$ are disjoint union of $F$. We define the transitions of terminal states as $\delta(u,l)=u$ for all $u\in F,l\in \Sigma$. A simple way to induce the reward function is 
% \begin{equation}
%     R(u,l)=
%     \begin{cases}
%     1, \text{if }u\not\in F^+,\delta(u,l)\in F^+;\\
%     0, \text{otherwise.}
%     \end{cases}
% \end{equation}

% \begin{define}
% Given an MDP $(S,A,p,\gamma)$, a set of propositional variables $\P$, a set of all possible labels $\Sigma\subseteq 2^\P$, and a labelling function $L:S\times A\times S\to \Sigma$, we define the extended \textbf{MDP with an RM} (MDPRM) \textit{w.r.t.} $\A=(U,u_0,F,\delta,R)$ as a tuple $\mathcal{M}=(S\times U,A,r',p',\gamma)$, where 
% % $S,A,p$ and $\gamma$ are defined as in an MDP, $\P$ is a set of propositional variables, $\Sigma\subseteq 2^\P$ is the set of all possible labels, $L:S\times A\times S\to \Sigma$ is a labelling function, and $\A$ is an RM.
% the reward function $r'$ is defined as $r'((s,u),a,(s',u'))=R(u,L(s,a,s'))$, and the probabilistic transition function $p'$ is defined as
% \begin{equation}
%     p'((s',u')\mid (s,u),a)=
%     \begin{cases}
%     p(s'\mid s,a),&\quad \text{if }u\not\in F,u'=\delta(u,L(s,a,s'))\\
%     1,&\quad \text{if }u\in F,u=u',s=s'\\
%     0,&
%     \quad \text{otherwise.}
%     \end{cases}
% \end{equation}
% \end{define}

\section{SLTL: Sequential Linear Temporal Logic}
In order to enable decomposition of sequential tasks, we add a new temporal operator $\sim$(\textit{then}) into the traditional LTL, resulting the \textit{Sequential LTL} (SLTL). Being compatible with LTL, SLTL provides a more succinct and flexible way to describe sequential tasks, and more importantly, enables us to exploit task modularity for more efficient transfer learning than LTL. For example, the task ``eventually complete $a$ then $b$'' is expressed as $\Diamond(a\land\bigcirc(\Diamond b))$ using LTL, but more directly as $(\Diamond a)\sim (\Diamond b)$ using SLTL. The latter expression can be decomposed into subtasks $\Diamond a$ and $\Diamond b$, the knowledge of which can be readily transferred to the learning of target task, say $(\Diamond b)\sim (\Diamond a)$. However, such straightforward manipulation cannnot be readily realized using the expression of LTL, since we cannot extract $\Diamond a$ from the LTL formula of $\Diamond(a\land\bigcirc(\Diamond b))$. Please refer to Appendix.A. for more details in the syntax and semantics of SLTL.
% state in introduction
% SLTL plays an important role in our lifelong learning framework, including describing the temporal tasks, automatically generating an RM and decomposing the target tasks to transfer knowledge from the learned tasks. 
%We first give the syntax and semantics of SLTL, and propose SLTL progression for formulas with the \textit{``then''} operator. We then define \textit{rational} SLTL in order to generate finite state RM without failure. At last we prove the laws of the temporal operator \textit{``then''} to provide more flexible specification of task modularity for wider possibility of knowledge transfer.

We prove the laws of the operator \textit{``then''} to provide different ways of decomposing a target task. 
\begin{theorem}
For any SLTL formulas $\varphi_1,\varphi_2,\varphi_3$, we have the associative law: $(\varphi_1\sim\varphi_2)\sim\varphi_3=\varphi_1\sim(\varphi_2\sim\varphi_3)$, and the following distribution laws: 
\begin{enumerate}[(i)]
    \item $(\varphi_1\sim\varphi_2)\land(\varphi_1\sim\varphi_3)=\varphi_1\sim(\varphi_2\land\varphi_3)$;
    \item $(\varphi_1\sim\varphi_3)\land(\varphi_2\sim\varphi_3)=(\varphi_1\land\varphi_2)\sim\varphi_3$;
    \item $(\varphi_1\sim\varphi_2)\lor(\varphi_1\sim\varphi_3)=\varphi_1\sim(\varphi_2\lor\varphi_3)$;
    \item $(\varphi_1\sim\varphi_3)\lor(\varphi_2\sim\varphi_3)=(\varphi_1\lor\varphi_2)\sim\varphi_3$.
\end{enumerate}
\end{theorem}
\begin{proof}
See Appendix.A.
\end{proof}

\begin{figure}[ht]
    \centering
    \includegraphics[width=0.9\linewidth]{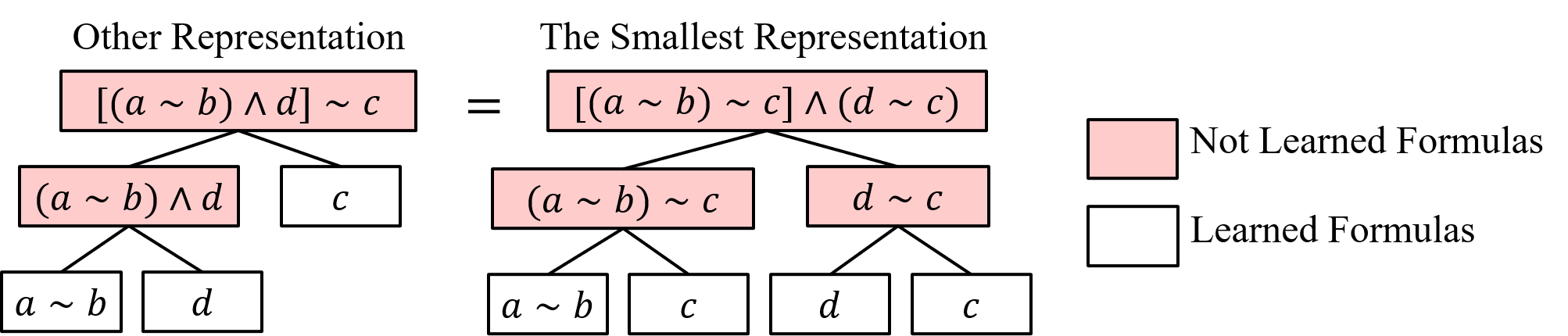}
    \caption{Two different ways of task decomposition using the operator laws.}
    \label{fig:operator_laws}
\end{figure}

The above operator laws enable various representations of a task, leading to different ways of task decomposition and thus diverse learning efficiency. Figure \ref{fig:operator_laws} gives an illustration: the two formulas on the top are equivalent representations of the same target task. Suppose that formulas $a\sim b,c,d$ (colored in white) indicate the subtasks that have been learned before (i.e., \textit{learned formulas}), while the remaining formulas (colored in red) indicate the new tasks that have not been learned yet. It is clear that the left presentation has fewer new subtasks than the right one. More formally, we define \textit{the smallest representation} of a task using the operator laws as follows:
\begin{define}
\label{def: smallest representation}
Let $\varphi_1,\varphi_2,\cdots,\varphi_n$ be the different representations of a target task, and $T_1,T_2,\cdots,T_n$ be their sub-formulas decomposed by the $\land,\lor,\sim$ operators. Given a set of learned formulas $T_M$, the smallest representation $\varphi^*$ of the target task is the one with the smallest number of subtasks that have been not learned before, \textit{i.e.,}
\begin{equation}
    \varphi^*=\varphi_i, \quad\text{where } i=\argmin_{i=1,2,\cdots,n} |T_i\setminus T_M|.
\end{equation}
\end{define}

The consideration of the smallest representation of a target task can be attributed to the fact that when facing a new task, it is more likely to learn faster if the task can be decomposed into fewer unknown subtasks. We will provide  experimental evaluations on this issue in Subsec.~\ref{sec: evaluation of smallest representation}. 

\section{LSRM: Lifelong Learning with SLTL \\and RM}
In this section, we introduce the LSRM method that combines the benefits of both SLTL and RM for efficient learning of logically specified tasks. Formally, LSRM maintains a \textit{memory} $M=(\A_M,T_M,\mathbf{Q}_M)$ for the agent over lifetime, where $\A_M=(U_M,u_0,F,\delta,R)$ is the memory RM, $T_M$ is the set of learned SLTL formulas from previous tasks, and $\mathbf{Q}_M=\{Q^\varphi(s,a)\mid \varphi\in U_M\}$ is the set of Q-functions corresponding to the states in $\A_M$. 
% Figure \ref{fig: LSRM} presents a
% n illustrative example including three phases over lifetime, where the agent learns to \textit{``deliver coffee to office''} and \textit{``deliver mail to office''} in the first phase; \textit{``deliver coffee and mail to office''} in the second phase, and other task in the third phases. 
In a learning phase of the lifetime, the agent is required to learn a set of target tasks, and different representations of these tasks can be generated by the operator laws. The smallest representation of the target task is then chosen according to the learned formulas $T_M$ in the memory, and then fed into Algorithm~\ref{alg:lifelong} as a set of target formula(s) $T$.
% The memory is then updated by extending the RM (see Subsec.\ref{sec:extension of RM}) and transferring knowledge from the learned Q-functions in blue to the new Q-functions colored in yellow (see Subsec.\ref{sec: evaluation of heuristic transferring}), as given by Algorithm \ref{alg:lifelong}.
Concretely, for each target formula $\varphi$ in $T$, LSRM first \textit{extends} the memory RM and returns its new extended states denoted as $U_{new}^\varphi$ (Line 2-3).
% by extracting formulas from $\varphi$ and storing the newly extracted formulas as extended  states of RM, as well as their transitions (Line 2), so that the extended RM includes all the portions of the input target task that account for future learning.
% After extension, the extended states are returned, denoted by $U_{new}^\varphi$ (Line 3). 
The learned Q-functions are then transferred from the memory to the new Q-function corresponding to each state in $U_{new}^\varphi$ (Line 4-5). Finally, the QRM algorithm is utilized to learn the target formulas (Line 8). When the agent starts to implement a task $\varphi\in T$ in an episode of QRM, the initial state of the memory RM is set to be $u_0=\varphi$. The set of learned task $T_M$ is then updated after the QRM learning procedure (Line 9). We show the details of extension of the memory RM and transferring knowledge in the following subsections. 

% \begin{figure}[ht]
%     \centering
%     \includegraphics[width=1\linewidth]{LSRM.png}
%     \caption{Illustration of the LSRM framework.}
%     \label{fig: LSRM}
% \end{figure}

\begin{algorithm}[ht]
\caption{The Update of the Memory in LSRM}
\label{alg:lifelong}
\textbf{Input}: Set of target formula(s) $T$, memory $M$ %=(\A_M,T_M,\mathbf{Q}_M)$

\textbf{Output}: New memory $M$ %=(\A_M,T_M,\mathbf{Q}_M)$
\begin{algorithmic}[1] %[1] enables line numbers
    % \STATE Choose the smallest representation for each target task
    \FORALL{target formula $\varphi\in T$}
        \STATE $M$.ExtendRM($\varphi$) 
        \STATE $U_{new}^\varphi\leftarrow M$.ReturnNewStates() 
        \FORALL{extended state $\psi\in U_{new}^\varphi$}
            \STATE $Q^{\psi}(s,a)\leftarrow M$.AcquireKnowledge($\psi$) 
        \ENDFOR
    \ENDFOR
    \STATE $M$.QRMRun($T$)
    \STATE $T_M\leftarrow (\cup_{\varphi\in T} U_{new}^\varphi) \cup T_M$
\end{algorithmic}
\end{algorithm}

\subsection{Extensions of Memory RM} \label{sec:extension of RM}

In this subsection, we discuss how to extend the memory RM with a target task prescribed by an SLTL formula, so that the extended RM includes all the modules of the target task for future learning. On one hand, the target formula can be iteratively progressed in order to extract sub-formulas using the SLTL progression, inspired by \textit{LTL Progression for Off-Policy Learning} \cite{toro2018teaching}. On the other hand, it can also be decomposed using the \textit{``then''} operator to extract sub-formulas accounting for the possibly encountered tasks in the future. Each extracted sub-formula by either progression or decomposition indicates a module of the target task. Whenever an extracted formula is new, \textit{i.e.,} not in the states $U_M$ of the original memory RM, the formula and its corresponding transitions will be stored as an extended state of the RM. 
%Figure \ref{fig:ExtendRM} gives an illustration of the extensions of memory RM when faced with target tasks $\varphi_1=(\Diamond c)\sim (\Diamond o)$:\textit{``deliver coffee to office'', $\varphi_2=(\Diamond c)\sim (\Diamond o)$:``deliver mail to office''} simultaneously at first, and then $\varphi_3=(\Diamond c\land \Diamond m)\sim (\Diamond o)$:\textit{``deliver coffee and mail to office''} later. The white nodes (rectangles) are original states, and the red nodes and arrows are extended states (\textit{i.e.,} new formulas) and their transitions, respectively. The memory RM (shown in left) is initialized as a terminal state without any transitions. Then we input $\varphi_1$ and $\varphi_2$ , resulting in three progressed formulas $\varphi_1,\varphi_2,\Diamond o$, and two formulas $\Diamond c, \Diamond m$ decomposed by $\varphi_1,\varphi_2$, respectively (shown in middle). The decomposed formulas accounts for the future tasks that contain the module of \textit{``getting coffee''} and \textit{``getting mail''}. After that, we input $\varphi_3$ and resulting in one progressed formula $\varphi_3$, and one formula $\Diamond c\land \Diamond m$ decomposed by $\varphi_3$ (shown in the right).

Formally, for each input target formula $\varphi$, the set of extracted formulas is initialized as $T_{ex}=\{\varphi\}$ and iteratively updated by:
 \begin{equation}
     T_{ex}\leftarrow T_{ex}\cup T_{\prog}\cup T_{dec}, 
 \end{equation}
where $T_{\prog}=\{\prog(\psi,l)\mid \psi\in T_{ex},l\in \Sigma\}$ is the set of formulas progressed by all the labels, and $T_{dec}=\{\psi_1\mid \psi\in T_{ex}, \psi=\psi_1\sim\psi_2\}$ is the set of sub-formulas decomposed by the $\sim$ operator. The iteration terminates if and only if the set $T_{ex}$ does not change anymore. 
% % The existence of such stable point is ensured by the rationality of the input formula $\varphi$, formally stated as Theorem xx. 
Denote the original memory RM as $\A_M=(U_M,u_0,F,\delta,R)$, then the new states are $U_{new}=T_{ex}\setminus U_M$. Therefore, the extended RM is $\A_M'=(U_M',u_0,F',\delta',R')$, where the states are $U_M'=U_M\cup U_{new}$; the transition function is $\delta'(\psi,l)=\delta(\psi,l)$ for original states $\psi\in U_M$ and $\delta'(\psi,l)=\prog(\psi,l)$ for new states $\psi\in U_{new}$; the terminal states are $F'=U_M'\cap \{\true,\false\}$; and the reward function $R'$ is defined as Eq.~(\ref{equ:reward function}).  Algorithm~\ref{alg:ExtendRM} gives the procedure of extending RM and Figure~\ref{fig:ExtendRM} plots an illustrative example of such a process.

\begin{figure*}[!tb]
    \centering
    \includegraphics[width=0.8\linewidth]{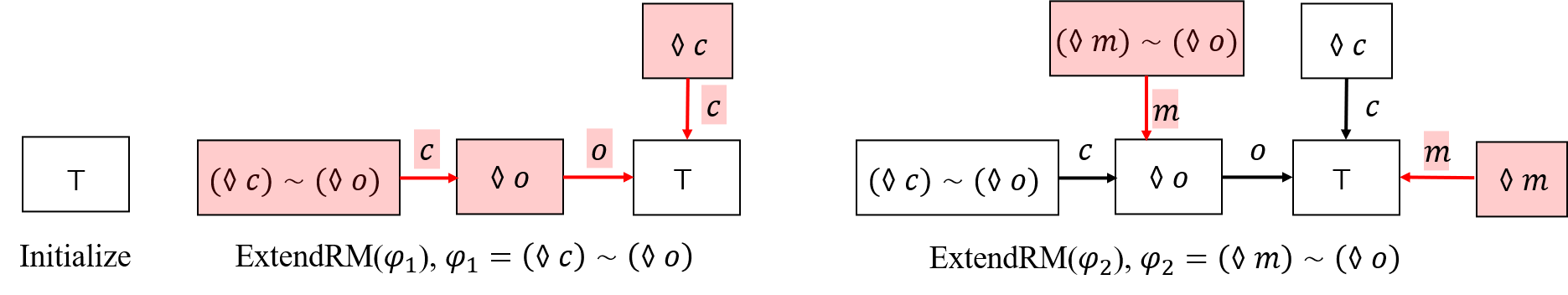}
    \caption{An illustration of extensions of memory RM when faced with target tasks $\varphi_1=(\Diamond c)\sim (\Diamond o)$:\textit{``deliver coffee to office''} at first then $\varphi_2=(\Diamond m)\sim (\Diamond o)$:\textit{``deliver mail to office''} later. The white nodes (rectangles) are original states, and the red nodes and arrows denote the extended states (\textit{i.e.,} new formulas) and their transitions, respectively. \textbf{Left:} The memory RM is first initialized as a terminal state without any transitions. \textbf{Middle:} When learning new task $\varphi_1$,  the task formula is first added to the RM as a new state, and decomposed by the $\sim$ operator, resulting in new formula $\Diamond c$. Then, the formula $\varphi_1$ is progressed~by each label $\{\},\{c\},\{m\},\{o\}$, resulting in sub-formula $\varphi_1,\Diamond o,\varphi_1,\varphi_1$, respectively. The new sub-formulas $\Diamond c,\Diamond o$ are not in the memory and thus added to the RM as new states. After that,  $\Diamond c, \Diamond o$ are  decomposed and progressed as above, ending up with no new formulas. \textbf{Right:} Similarly, after adding the task formula $\varphi_2$ as new state, only a new decomposed formula $\Diamond m$ is added without no new progressed formulas.}
    \label{fig:ExtendRM}
\end{figure*}

%The implementation is shown in Algorithm \ref{alg:ExtendRM}. If the target SLTL formula $\varphi$ is contained in the memory RM, then there is no need to extend (Line 1-3). Otherwise, we maintain a work queue to store the new formulas, initialized as the input formula $\varphi$ (Line 4). Whenever the queue is not empty (Line 5), we pop a formula $\psi$ from the queue (Line 6) and firstly decompose it by the $\sim$ operator (Line 7). If the decomposed formula is new, it is  added to the RM and the queue (Line 8). After that, the formula $\psi$ is progress by each label in the label set $\Sigma$ (Line 10-11). Similarly, the progressed formula is added to the RM and the queue whenever it is new (Line . Finally, the transitions among new states are stored in the memory RM, which are consistent with SLTL progression (Line 15).  

\begin{algorithm}[ht]
\caption{ExtendRM($\varphi$)}
\label{alg:ExtendRM}
\textbf{Input}: Memory $M$, the target SLTL formula $\varphi$

\textbf{Output}: New memory $M$
\begin{algorithmic}[1] %[1] enables line numbers
    \IF{$\varphi\in U_M$}
        \STATE \textbf{return}
    \ELSE
        \STATE Initialize queue=[$\varphi$],  $U_M\leftarrow U_M\cup \{\varphi\}$
        \WHILE{\textbf{not} queue.empty()}
        \STATE $\psi\leftarrow$ queue.pop()
        \IF{$\psi=\psi_1\sim\psi_2$ \textbf{and} $\psi_1\not\in U_M$}
            \STATE $U_M\leftarrow U_M\cup \{\psi_1\}$, queue.append($\psi_1$)
        \ENDIF
        \FORALL{$l\in \Sigma$}
            \STATE $\psi'=\prog(\psi,l)$
            \IF{$\psi'\not\in U_M$}
                \STATE $U_M\leftarrow U_M\cup \{\psi'\}$, queue.append($\psi'$)
            \ENDIF
            \STATE Store transition  $\delta(\psi,l)=\psi'$ for $\A_M$
        \ENDFOR
        \ENDWHILE
    \ENDIF
\end{algorithmic}
\end{algorithm}

\subsection{Acquisition of Knowledge from the Memory}
\label{sec: knowledge transfer}
The key idea to acquire knowledge, \textit{i.e.}, Q-functions from the memory for a target formula $\varphi$, is to iteratively decompose the target formula $\varphi$ into sub-formulas until all the sub-formulas are in the learned tasks $T_M$, such that the Q-function of the target formula $\varphi$ can be initialized using the composed Q-values of the learned sub-formulas. %The recursive procedure of knowledge acquisition is shown in Algorithm \ref{alg:Transfer}. 
If the input formula $\psi$ has been learned, then its corresponding Q-function is returned directly. If $\psi$ has not been learned but can be decomposed into two sub-formulas $\psi_1\circ\psi_2$, where the operator $\circ$ is $\land,\lor$ or $\sim$, then the process first (\textit{i}) recursively transfers Q-functions from the memory to the Q-functions of $\psi_1,\psi_2$, which are denoted by $Q_1,Q_2$, respectively; and then (\textit{ii}) composes the values of $Q_1,Q_2$ and returns the results. Finally, if the target formula is neither learned nor decomposable, its Q-function is then randomly initialized.

In~\cite{van2019composing}, two composition methods (\textit{i.e.}, the \textit{maximum and average} composition) were proposed in entropy regularised
RL and proved to achieve optimal composition value function for the operator \textit{``or'', ``and''}, respectively. To consider the sequential operator \textit{``then''}, two new composition methods (\textit{i.e.}, the \textit{left} and \textit{right} composition) are proposed here. The four Q-value composition methods are listed as follows:

\begin{itemize}
    \item \textit{Average Composition}: $Q^{\psi_1\circ \psi_2}(s,a)\leftarrow  (Q^{\psi_1}(s,a)+Q^{\psi_2}(s,a))/2$;
    \item \textit{Max Composition}:  $Q^{\psi_1\circ \psi_2}(s,\cdot)\leftarrow
    \begin{cases}
    Q^{\psi_1}(s,\cdot), \text{if }\max_a Q^{\psi_1}(s,a)>\max_a Q^{\psi_2}(s,a);\\
    Q^{\psi_2}(s,\cdot), \text{otherwise};
    \end{cases}$
    \item \textit{Left Composition}: $Q^{\psi_1\circ \psi_2}(s,a)\leftarrow Q^{\psi_1}(s,a)$;
    \item \textit{Right Composition}: $Q^{\psi_1\circ \psi_2}(s,a)\leftarrow Q^{\psi_2}(s,a)$.
\end{itemize}

\section{Experimental Results} \label{sec: experiment}

In this section, we first evaluate the performance of different Q-value composition methods for each operator, and then verify if the learning performance can be improved when choosing the smallest representation of a task. Finally, we implement LSRM in learning a series of SLTL tasks and evaluate its performance, compared to some direct learning methods. All the experiments are evaluated in the \textsc{OfficeWorld} domain \cite{icarte2018using} and the \textsc{MineCraft} domain \cite{andreas2017modular}. 
%Both of the two domains are grid worlds and the agent can move in the four cardinal directions. Some of the grids are marked by a symbol $p\in \P$, and a proposition variable $p$ is true if and only if the agent is located at that marked grid. 
% Therefore, the label set of these domains is $\Sigma=\{\{p\}\mid p\in \P\}$, since there is at most one symbol marked in one grid. 
The parameters are set to be $\epsilon=0.1, \gamma=0.9$ and $\alpha=1$ in both domains. We report the performance by calculating the average steps to complete the target task(s) throughout the training process. We compare LSRM with the QRM and QRM with reward shaping (QRM+RS) \cite{toro2020reward}, and the final results are averaged over 20 independent trails. 

% \begin{figure}[ht]
%     \centering
    
%     \subfigure[\textsc{MineCraft}]{\label{fig:MineCraft}
%     \includegraphics[width=0.2\linewidth]{MineCraft.png}}
%     \caption{Two benchmark domains.}
%     \label{fig:domains}
% \end{figure}

\subsection{Evaluation of Different Composition Methods}
\label{sec: evaluation of heuristic transferring}
% introduction
% how to test
We sample tasks from a set of source tasks \textit{i.e.,} learned SLTL formulas, and compose them by the operators \textit{``and'', ``or''} and \textit{``then''} to define the target tasks. 
%For instance, if two of the source tasks are expressed by $\varphi_1$  and $\varphi_2$, then the target tasks are $\varphi_1\land \varphi_2, \varphi_1\lor\varphi_2$ and $\varphi_1\sim\varphi_2$. 
For more detailed information of the source tasks and the target tasks, see Appendix.B. The RM of these target tasks is then automatically generated using the method of extending RM in Subsec.~\ref{sec:extension of RM}. We then initialize the Q-functions of the target tasks by composing Q-values of the source tasks using the four value composition methods in Subsec.~\ref{sec: knowledge transfer}. Finally, QRM is applied to learn each target task. 

\begin{figure*}[ht]
    \centering
    \subfigure[Results in the \textsc{OfficeWorld} domain.]{\label{fig:E1office}
    \includegraphics[width=0.8\linewidth]{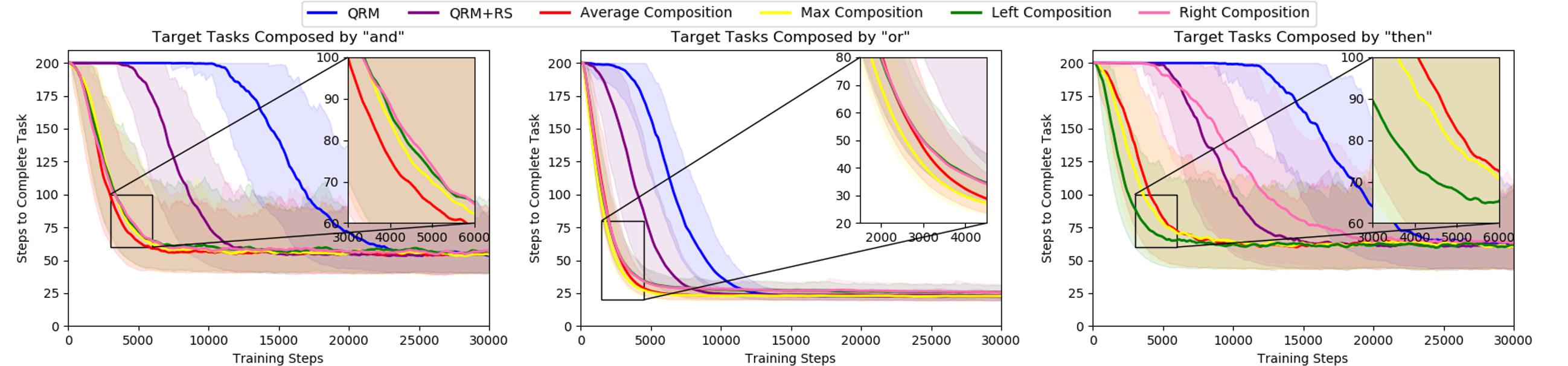}}
    \subfigure[Results in the \textsc{MineCraft} domain.]{\label{fig:E2craft}
    \includegraphics[width=0.8\linewidth]{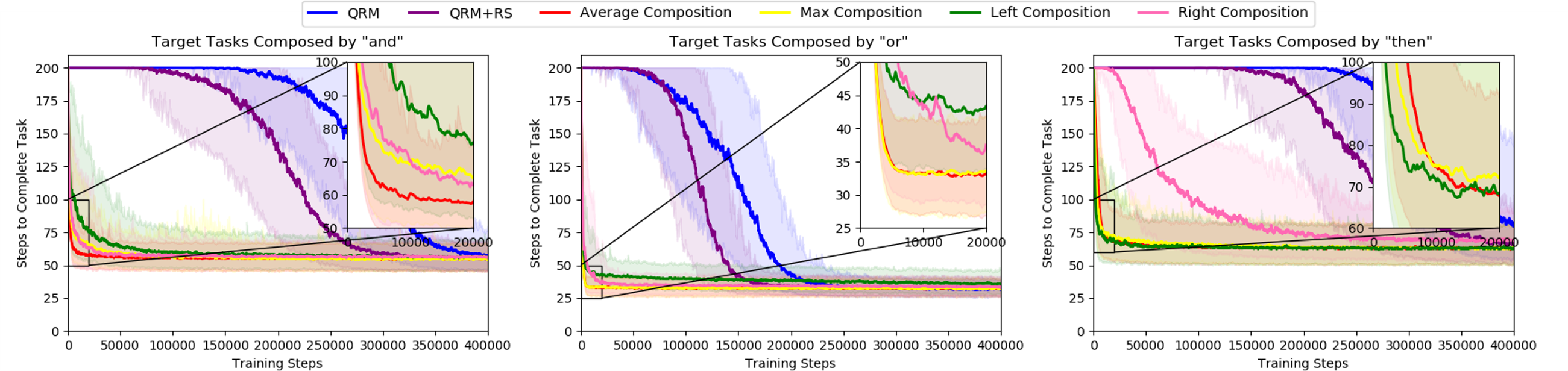}}
    \caption{Evaluation of different Q-value composition methods in the two domains. The learning efficiency of target tasks composed by the \textit{``and'', ``or''} and \textit{``then''} operators are shown in the left, middle and right column, respectively.}
    \label{fig:evaluation Q-value composition}
\end{figure*}
% result, explanation 
The experimental results are shown in Figure \ref{fig:evaluation Q-value composition}. In both domains, all the methods outperform QRM and QRM+RS by a large margin, except the \textit{Right Composition} for target tasks composed by \textit{``then''} in the \textsc{OfficeWorld} domain. The \textit{Average Composition, Max Composition, Left Composition} perform best when the target tasks are composed by the \textit{``and'', ``or'', ``then''} operators in both domains, respectively. For the \textit{``and''} operator, if an action $a$ is optimal under state $s$ for both learned tasks $\varphi_1$ and $\varphi_2$, then the optimal action of the averaged Q-values under $s$ is $a$. Thus, choosing $a$ under $s$ helps completing the target task $\varphi_1\land\varphi_2$. For the \textit{``or''} operator, the best way to complete the target task $\varphi_1\lor\varphi_2$ is to complete one of the source tasks $\varphi_1,\varphi_2$ with less steps. If $\max_{a'}Q^{\varphi_1}(s,a')> \max_{a'}Q^{\varphi_2}(s,a')$ under state $s$, then following the policy of $Q^{\varphi_1}$ helps completing the target task more quickly. The above results and analysis are consistent with the previous results in~\cite{van2019composing}, although their focus is on general RL tasks that are sampled from the same MDP distribution but only differ in their reward functions, while we focus on logically specified tasks that are not necessarily Markovian. Finally, for the \textit{``then''} operator, the agent has to complete $\varphi_1$ first when it implements the target task $\varphi_1\sim\varphi_2$. Therefore, following the policy of $Q^{\varphi_1}$ helps learning the sequential target task.

%According to the above results, we choose the \textit{Average Composition, Max Composition, Left Composition} when the target task is composed by \textit{``and'', ``or'', ``then''} operator, respectively, and apply the corresponding methods in the LSRM framework. 

\begin{figure}[ht]
    \centering
    \subfigure[ \textsc{OfficeWorld}.]{\includegraphics[width=0.48\linewidth]{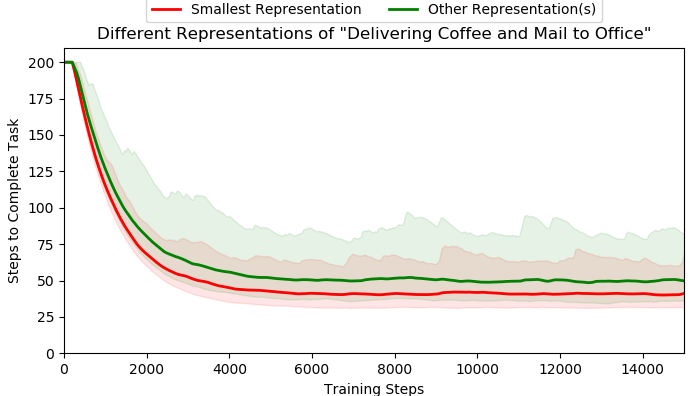}}
    \subfigure[ \textsc{MineCraft}]{\includegraphics[width=0.48\linewidth]{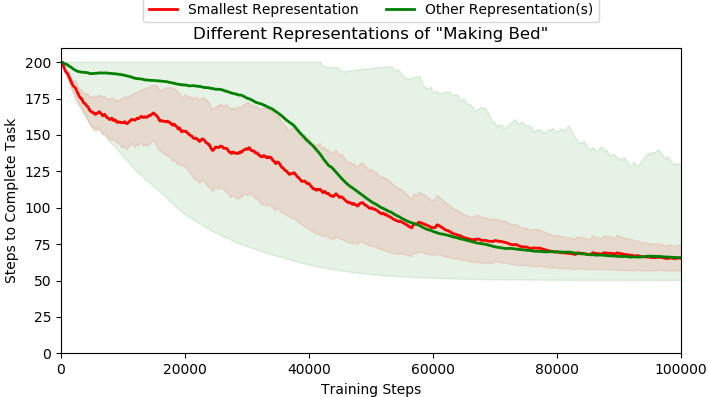}}
    \caption{Evaluation of different representations of tasks.}
    \label{fig: evaluation different representation}
\end{figure}

\subsection{Evaluation of Different Task Representations}
\label{sec: evaluation of smallest representation}
We evaluate the performance of learning with different formulas of the same target task. The target tasks are \textit{``delivering coffee to office avoiding furniture''} in the \textsc{OfficeWorld} domain and \textit{``making bed''} in the \textsc{MineCraft} domain.
% $(\varphi_c\land \varphi_m)\sim \varphi_o$ and $(\varphi_c\sim\varphi_o)\land(\varphi_m\sim\varphi_o)$, where $\varphi_c=(\neg *)\until c, \varphi_m=(\neg *)\until m,\varphi_o=(\neg *)\until o$. 
The detailed representations of these target tasks are shown in Appendix.B. From the results in Figure \ref{fig: evaluation different representation}, we can see that choosing the smallest representation with the fewest sub-formulas that have not been learned can improve the learning performance against other presentations. This result is reasonable as fewer unknown subtasks would bring less uncertainty during the value composition or initialization process, and thus facilitate learning of the target task.

\subsection{Evaluations of LSRM in LRL Tasks}
\label{sec: evaluation of LSRM}
Finally, we evaluate LSRM in learning a series of tasks using the best corresponding composition method for each operator as given in Subsec.~5.1 (denoted as LSRM-best) and the worst corresponding composition method (denoted as LSRM-worst), and compare them to QRM and QRM+RS. In the \textsc{OfficeWorld} domain, we define a sequence of 6 tasks, and require the agent to learn 2 tasks per phase, where each phase contains 30,000 training steps. For example, in the first phase of 30,000 steps, the agent repeatedly learns two tasks \textit{``deliver coffee to place A avoiding furniture''} and \textit{``deliver mail to place B avoiding furniture''} in order. In the \textsc{MineCraft} domain, we adopt 10 tasks defined by \cite{andreas2017modular}, and also require the agent to learn 2 tasks per phase, where each phase contains 400,000 training steps. For more detailed information of the tasks and illustrations of the LSRM processes, see Appendix.B and Appendix.C, respectively. Experimental results are shown in Figure \ref{fig:evaluation LSRM}. As can be seen, in both domains, LSRM-best outperforms the QRM and QRM+RS baselines except for the first phase, where there is no previous knowledge to be transferred. LSRM-worst performs slightly worse than LSRM-best, but still outperforms QRM baseline since it still leverages some kind of knowledge from previous tasks, even if this knowledge transfer might not be optimal. 
% Using the smallest representations improves the performance of LSRM in the \textsc{MineCraft} domain, where the tasks are more complex and the representations should be selected more carefully.  However, the reward shaping method has negative impact on LSRM, since it is sensitive to the initialization of Q-functions. 

\begin{figure*}[ht]
    \centering
    \subfigure[Results in the \textsc{OfficeWorld} domain.]{\label{fig:E3office}
    \includegraphics[width=0.8\linewidth]{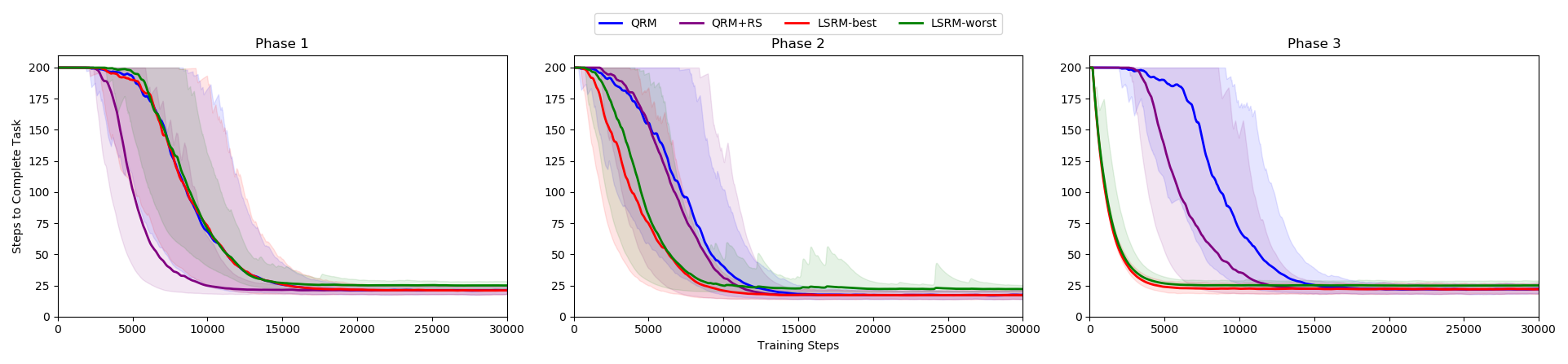}}
    \subfigure[Results in the \textsc{MineCraft} domain.]{\label{fig:E3craft}
    \includegraphics[width=0.8\linewidth]{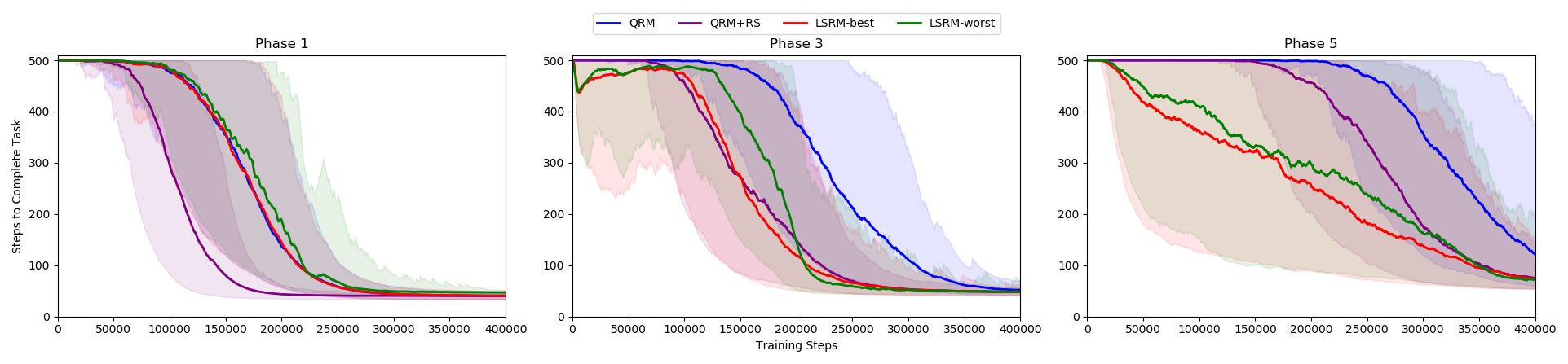}}
    \caption{Evaluations of the LSRM in each phase in the \textsc{OfficeWorld} domain, and phase 1,3,5 in the \textsc{MineCraft} domain. }
    \label{fig:evaluation LSRM}
\end{figure*}

\section{Related Work}

% lifelong learning
Lifelong learning (or continual learning, multi-task learning) has received a rising interest in recent years, due to its potential to reduce agents' training time in dynamic environments. 
%Compared to the vast literature of lifelong learning in deep learning community \cite{li2017learning}, LRL is relatively less studied.  
Abel \textit{et al.}~\cite{abel2018policy} proposed a transfer method to realize optimal initialization of an agent’s policy or value function in LRL. Garcia \textit{et al.}~\cite{garcia2019meta} leveraged advice from previously learned tasks to enable more efficient exploration in target tasks. Ammar \textit{et al.}~\cite{ammar2015autonomous} proposed an LRL algorithm that supports efficient cross-domain transfer between tasks from different domains. An option-discovery method \cite{brunskill2014pac} was proposed to facilitate learning by transferring the options with high sample efficiency. Other studies resorted to techniques including active learning with network consolidation (compression)  \cite{schwarz2018progress}, sub-network freezing \cite{rusu2016progressive}, or rehearsal of old data via experience replay \cite{isele2018selective,rolnick2018experience}. However, these works normally assume that the series of tasks share some similar structure in MDP components such as rewards or transition probabilities, therefore it is hard for them to handle non-Markovian logically specified tasks as we did in this paper. 

%Some studies  \cite{tasse2020boolean,tasse2020logical,van2019composing} also investigated logical task composition in LRL. However, their focus is on the composition effectiveness given particular logical operators, while we address the automatic decomposition and knowledge transfer in learning more complex sequential tasks that are specified by LTL. 

% transfer learning+temporal logic
%Imitation learning enables users to specify tasks by providing demonstrations of the desired task [18, 1, 26, 20, 10]. However, in many settings, it may be easier for the user to directly specify the task—e.g., when programming a warehouse robot, it may be easier to specify waypoints describing paths the robot should take than to manually drive the robot to obtain demonstrations \cite{zhang2020survey}

Formal language such as LTL allows the user to specify task constraints over sequences of events happening over time. A large number of studies have incorporated the rich expressive power of logic formulas into RL. Icarte \textit{et al.}~\cite{icarte2017using} considered the use of advice expressed in LTL to guide exploration in a model-based RL algorithm. Bozkurt \textit{et al.}~\cite{bozkurt2020control} applied  LTL objectives to robotic learning. Li \textit{et al.} proposed a formal method approach to RL that makes the reward generation process interpretable~\cite{li2019formal}, and proposed a variant of LTL to specify a reward function that can be optimized using RL~\cite{li2017reinforcement}. De Giacomo \textit{et al.}~\cite{de2019foundations,de2020imitation} leveraged LTL as constraints (i.e., restraining bolts) in RL. Gao \textit{et al.}~\cite{gao2019reduced} 
used LTL to specify the unknown transition probabilities for RL. However, all these studies do not target at LRL problems.  

A number of studies have considered as well logic formulas for more efficient skill composition and knowledge transfer in RL. Yuan \textit{et al.}~\cite{yuan2019modular} proposed a modular RL approach to satisfy one-shot TL specifications in continuous state and action spaces. Li  \textit{et al.}~\cite{li2017automata} also used the modularity of TL and automata together with hierarchical RL for skill composition to generalise from trained sub-tasks to complex specifications. Xu  \textit{et al.}~\cite{xu2019transfer} proposed the \textit{Metric Interval Temporal Logic} (MITL)  to specify temporal tasks and transfer knowledge between \textit{logically transferable} tasks.  
Leon \textit{et al.}~\cite{leon2020systematic} utilized an extended observation to be fed into the network to encode new high-level knowledge.
%The concept of \textit{logically transferable} was proposed to measure similarity between temporal tasks. 
The authors in~\cite{jothimurugan2019composable} proposed a formal language for specifying complex control tasks.
%, and compiled task specifications into a reward function in continuous state and action for automatic reward shaping. 
Other works reveal the relationships between tasks by constructing \textit{Markov Logic Network} (MLN) \cite{torrey2009policy,mihalkova2007mapping} and investigate the composition effectiveness given particular logical operators \cite{tasse2020boolean,tasse2020logical,van2019composing}. Compared to these methods that the source of sub-tasks should be specified a priori and the focus is on bottom-up composition effectiveness, our work is able to learn and decompose logically complex tasks automatically in a lifelong memory.

Finally, RM has been used to solve problems from various perspectives in RL, such as robotics training \cite{camacho2020disentangled}, encoding a team's task in multi-agent systems \cite{neary2020reward} and solving partially observable RL problems \cite{icarte2019learning}. Unlike these existing studies, our work features a dynamic growing RM for continuous learning of future target formulas.

%\cite{de2020temporal} llowing a reward designer to assign rewards based on temporary/permanent satisfaction of a temporal formula.

\section{Conclusion}
This paper presents an LRL method that takes advantage of the rich expressive power of temporal logic formulas for more flexible task specification and decomposition, and the knowledge transfer capability of RM for more efficient learning. 
Through storing and leveraging high-level knowledge in a memory, LSRM is able to achieve systematic out-of-distribution generalisation in tasks that follow the specifications of formal languages.
%, thus addressing the challenging issue of \textit{catastrophic forgetting} in lifelong learning~\cite{li2017learning}. 
Results in two benchmark domains show that LSRM improves the sample efficiency by a large margin compared to direct learning methods.
In this paper, we evaluate the proposed method in grid-world domains with discrete actions.
In the future, we will extend our approach to more complex domains, \textit{e.g.}, with continuous states and/or actions. 

%We have proposed a language for formally specifying control tasks and an algorithm to learn policies to perform tasks specified in the language. Our algorithm first constructs a task monitor from the given specification, and then uses the task monitor to assign shaped rewards to runs of the system. Furthermore, the monitor state is also given as input to the controller, which enables our algorithm to learn policies for non-Markovian specifications. Finally, we implemented our approach in a tool called SPECTRL, which enables the users to program what the agent needs to do at a high level; then, it automatically learns a policy that tries to best satisfy the user intent. We also demonstrate that SPECTRL can be used to learn policies for complex specifications, and that it can outperform state-of-the-art baselines.

% \begin{quote}
% \noindent AAAI Press\\
% 2275 East Bayshore Road, Suite 160\\
% Palo Alto, California 94303\\
% \textit{Telephone:} (650) 328-3123\\
% \textit{E-mail:} See the submission instructions for your particular conference or event.
% \end{quote}

%\bibliography{aaai22.bib}

\appendix

\section{Full Definitions and Proofs of SLTL}
The syntax of SLTL formulas differs from LTL in terms of the \textit{``then''} operator, \textit{i.e.,}
\begin{equation}
    \varphi:=p\mid \neg \varphi \mid \varphi\land\psi \mid \bigcirc\varphi \mid \varphi\until\psi \mid \varphi\sim\psi,\quad p\in \P.
\end{equation}
% The semantic of SLTL are defined over \textit{label sequences}, where each \textit{label} $l_i\subseteq \P$ is a set of propositions that hold true in the environment at the time $t=i$. 
% For example, in the \textsc{OfficeWorld} domain, the agent gets coffee and mail at time $t=i$, then a label $l_i=\{c,m\}$ is detected. 
% 
% When an episode terminates, a finite label sequence is detected by the agent, says $\lambda=l_0l_1\cdots l_k$. If the value of an SLTL formula $\varphi$ is true determined by $\lambda$, \textit{i.e.} the agent completes the task specified by $\varphi$, we denote it by $\lambda\models\varphi$. Otherwise, we denote $\lambda\not\models\varphi$. It is formally defined as follows:
The semantic of the operator \textit{``then''} is defined as
\begin{align}
\lambda\models \varphi\sim\psi 
\iff 
\begin{cases}
\exists j\geq 0,\lambda^{0:j}\models\varphi,\lambda^{j+1}\models \psi, \\
\qquad\text{and }\forall i<j,\lambda^{0:i}\not\models \varphi, &\text{if }\varphi\neq\true;\\
\lambda\models \psi, &\text{if } \varphi=\true,
\end{cases}  
\end{align} 
where the notation $\lambda^{i:j}$ denotes the sub sequence $l_il_{i+1}\cdots l_j$.
% The definition above implies that if an agent completes a sequential task \textit{``task 1 then task 2''}, it needs to complete task 1 at first, then starts to complete task 2.

% If $i>k$ or $i>j$, we let $\lambda^i=\epsilon$ or $\lambda^{i:j}=\epsilon$, where $\epsilon$ is the empty label sequence. We also define $\epsilon\models\true$ and $\epsilon\not\models \varphi$ for all $\varphi\neq \true$.

The SLTL progression of formulas composed by the \textit{``then''} operator is defined as  
\begin{equation}
\prog(\varphi\sim\psi,l)=
  \begin{cases}
  \prog(\varphi,l)\sim \psi,\quad &\text{if } \varphi\neq\true;\\
  \prog(\psi,l),\quad &\text{if }\varphi=\true,
  \end{cases}
\end{equation}
where $\varphi,\psi$ are SLTL formulas and $l$ is a label.
% Progression of formulas including other operators can be induced by their definition. For example, from the definition of $\lor$ and $\Diamond$, we have $\prog(\varphi\lor\psi,l)=\prog(\varphi,l)\lor\prog(\psi,l)$ and $\prog(\Diamond\varphi,l)=\prog(\varphi,l)\lor \Diamond\varphi$.
The following theorem theoretically ensures that the progression of SLTL is well defined.
\begin{theorem}
\label{th:ltl progression}
For all SLTL formula $\varphi$ and a label sequence denoted by $\lambda$, we have $\lambda^i\models \varphi$ if and only if $\lambda^{i+1}\models \prog(\varphi,l_i)$ for all $i\geq 0$.
\end{theorem}
\begin{proof}
It has been proved for LTL formulas. Therefore, we only need to prove that $\lambda^i\models \varphi\sim\psi\iff \lambda^{i+1}\models \prog(\varphi\sim\psi,l_i)$. If $\varphi_1\neq \true$, then $\lambda^i\models \varphi_1\sim \varphi_2 \iff \exists j\geq i,\lambda^{i:j}\models \varphi_1,\lambda^{j+1}\models \varphi_2, \forall t<j, \lambda^{i:t}\not\models\varphi_1 \iff \lambda^{i+1:j}\models \prog(\varphi_1,l_i)$ and $\lambda^j\models\varphi_2, \forall  t<j, \lambda^{i+1:t}\not\models\prog(\varphi_1,l_i)\iff \lambda^{i+1}\models \prog(\varphi_1,l_i)\sim\varphi_2=\prog(\varphi_1\sim\varphi_2,l_i)
$. If $\varphi_1=\true$, then $\lambda^i\models \varphi_1\sim \varphi_2=\varphi_2 \iff \lambda^{i+1}\models \prog(\varphi_2,l_i)=\prog(\varphi_1\sim\varphi_2,l_i)$.
\end{proof}

% We should note that not all the SLTL formulas describe a task rationally. For example, $\neg\Diamond *$ describe the task \textit{``never touch the furniture''} that can only be broke down without a positive goal for the agent. 
% The formula $(A\sim \Diamond B)\until C$ derives infinite amount of progressed formulas, not suitable for generating an RM with finite states. Therefore, we restrict our consideration to \textit{rational} SLTL formulas, representing finitely achievable tasks for the agent.

% \begin{define}
% An SLTL formula $\varphi$ is said to be rational, if there exists finite label sequence $\lambda_0$ such that $\prog(\varphi,\lambda_0)=\true$ and all the progressed formulas $\{\prog(\varphi,\lambda)\mid \lambda\in \Sigma^*\}$ is a finite set, where the notation $\prog(\varphi,\lambda)$ means $\prog(\cdots,\prog(\prog(\varphi,l_0),l_1),\cdots,l_k)$ for $\lambda=l_0l_1\cdots l_k$.
% \end{define}

Finally, we give the proofs of operator laws in the following.

\begin{theorem}
For any SLTL formulas $\varphi_1,\varphi_2,\varphi_3$, we have associative law: $(\varphi_1\sim\varphi_2)\sim\varphi_3=\varphi_1\sim(\varphi_2\sim\varphi_3)$, and distribution laws (i)$(\varphi_1\sim\varphi_2)\land(\varphi_1\sim\varphi_3)=\varphi_1\sim(\varphi_2\land\varphi_3)$; (ii)$(\varphi_1\sim\varphi_3)\land(\varphi_2\sim\varphi_3)=(\varphi_1\land\varphi_2)\sim\varphi_3$; (iii)$(\varphi_1\sim\varphi_2)\lor(\varphi_1\sim\varphi_3)=\varphi_1\sim(\varphi_2\lor\varphi_3)$; and (iv)$(\varphi_1\sim\varphi_3)\lor(\varphi_2\sim\varphi_3)=(\varphi_1\lor\varphi_2)\sim\varphi_3$.
\end{theorem}
\begin{proof}
We first prove the associative law. We have
$\lambda\models (\varphi_1\sim\varphi_2)\sim\varphi_3\iff \exists 0\leq i<j, \lambda^{0:i}\models \varphi_1, \lambda^{i+1:j}\models \varphi_2,\lambda^j\models \varphi_3 \iff \lambda\models \varphi_1\sim(\varphi_2\sim\varphi_3)$. Hence, $(\varphi_1\sim\varphi_2)\sim\varphi_3=\varphi_1\sim(\varphi_2\sim\varphi_3)$.

We then prove the distribution law (i). $\lambda\models (\varphi_1\sim\varphi_2)\land(\varphi_1\sim\varphi_3)\iff \exists i,j\geq 0, \lambda^{0:i}\models \varphi_1, \lambda^{i+1}\models \varphi_2,\lambda^{0:j}\models \varphi_1, \lambda^{j+1}\models \varphi_2$ and$\forall k<\min\{i,j\}, \lambda^{0:k}\not\models \varphi_1 \iff \lambda^{0:\min\{i,j\}}\models \varphi_1, \lambda^{\min\{i,j\}}\models \varphi_2\land\varphi_3 \iff \lambda\models \varphi_1\sim(\varphi_2\land\varphi_3)$. Hence, $(\varphi_1\sim\varphi_2)\land(\varphi_1\sim\varphi_3)=\varphi_1\sim(\varphi_2\land\varphi_3)$. The proofs of (ii)-(iv) are similar.
\end{proof}

\section{Experimental Settings}
The source tasks and target tasks in Section \ref{sec: evaluation of heuristic transferring} are listed in Table \ref{tab:source tasks} and \ref{tab:target tasks}, respectively.

\begin{table*}[ht]
    \centering
    \begin{tabular}{|c||c|c|}
    \hline
    Domain & Source Task & SLTL Formula \\\hline
    \multirow{7}{*}{\textsc{OfficeWorld}} & \textit{deliver coffee to office avoiding furniture} & $\varphi_1=(\neg *\until c)\sim (\neg *\until o)$
    \\\cline{2-3}
    & \textit{deliver mail to office avoiding furniture} & $\varphi_2=(\neg *\until m)\sim (\neg *\until o)$\\\cline{2-3}
    & \textit{go to B then office avoiding furniture} & $\varphi_3=(\neg *\until B)\sim (\neg *\until o)$\\\cline{2-3}
    & \textit{go to B then C avoiding furniture} & $\varphi_4=(\neg *\until B)\sim (\neg *\until C)$\\\cline{2-3}
    & \textit{go to office then B avoiding furniture} & $\varphi_5=(\neg *\until o)\sim (\neg *\until B)$\\\cline{2-3}
    & \textit{get mail then go to D avoiding furniture} & $\varphi_6=(\neg *\until m)\sim (\neg *\until D)$\\\cline{2-3}
    & \textit{get mail then go to A avoiding furniture} & $\varphi_7=(\neg *\until m)\sim (\neg *\until A)$\\\hline\hline
    \multirow{6}{*}{\textsc{MineCraft}} & \textit{make plank} & $\psi_1=(\Diamond a)\sim (\Diamond b)$\\\cline{2-3}
    & \textit{make stick} & $\psi_2=(\Diamond a)\sim (\Diamond c)$\\\cline{2-3}
    & \textit{make cloth} & $\psi_3=(\Diamond d)\sim (\Diamond e)$\\\cline{2-3}
    & \textit{make rope} & $\psi_4=(\Diamond d)\sim (\Diamond b)$ \\\cline{2-3}
    & \textit{make bridge} & $\psi_5=(\Diamond a \land \Diamond f)\sim (\Diamond e)$\\\cline{2-3}
    & \textit{make shears} & $\psi_6=(\Diamond a\land \Diamond f) \sim (\Diamond c)$ \\\hline
    \end{tabular}
    \caption{Source tasks in the two domains.}
    \label{tab:source tasks}
\end{table*}

\begin{table*}[ht]
    \centering
    \begin{tabular}{|c||c|c|}
    \hline
    Domain & Target Task & SLTL Formula \\\hline
    \multirow{7}{*}{\textsc{OfficeWorld}} & \textit{complete source task 1 and 2} & $\varphi_1\land\varphi_2$
    \\\cline{2-3}
    & \textit{complete source task 4 and 6} & $\varphi_4\land\varphi_6$\\\cline{2-3}
    & \textit{complete source task 4 or 5} & $\varphi_4\lor\varphi_5$\\\cline{2-3}
    & \textit{complete source task 4 or 7} & $\varphi_4\lor \varphi_7$\\\cline{2-3}
    & \textit{complete source task 4 then 5} & $\varphi_4\sim \varphi_5$\\\cline{2-3}
    & \textit{complete source task 4 then 6} & $\varphi_4\sim \varphi_6$\\\cline{2-3}
    & \textit{complete source task 5 then 6} & $\varphi_5\sim \varphi_6$\\\hline\hline
    
    \multirow{6}{*}{\textsc{MineCraft}} & \textit{complete source task 1 and 3} & $\psi_1\land\psi_3$
    \\\cline{2-3}
    & \textit{complete source task 4 and 5} & $\psi_4\land\psi_5$\\\cline{2-3}
    & \textit{complete source task 4 and 6} & $\psi_4\land\psi_6$\\\cline{2-3}
    & \textit{complete source task 1 or 3} & $\psi_1\lor\psi_3$\\\cline{2-3}
    & \textit{complete source task 2 or 3} & $\psi_2\lor \psi_3$\\\cline{2-3}
    & \textit{complete source task 2 then 3} & $\psi_2\sim\psi_3$\\\cline{2-3}
    & \textit{complete source task 5 then 6} & $\psi_5\sim\psi_6$\\\hline
    \end{tabular}
    \caption{Target tasks in the two domains.}
    \label{tab:target tasks}
\end{table*}

In Section \ref{sec: evaluation of smallest representation}, the target task is \textit{``delivering coffee and mail to office avoiding furniture''} in the \textsc{OfficeWorld} domain. Its smallest representation is $[(\neg *\until c)\sim (\neg *\until o)]\land[(\neg *\until m)\sim (\neg *\until o)]$, and the other representation is $[(\neg *\until c)\land(\neg *\until m)]\sim (\neg *\until o)$ (see Figure~\ref{fig: representation office}). In the \textsc{MineCraft} domain, the target task is \textit{``making bed''}, and its smallest representation is 
$[(\Diamond a\sim\Diamond b) \land \Diamond d]\sim (\Diamond c)$, while the other representations are $[(\Diamond a\sim \Diamond b)\sim \Diamond c]\land (\Diamond d \sim \Diamond c)$ and $[\Diamond a\sim (\Diamond b\sim \Diamond c)]\land (\Diamond d \sim \Diamond c)$ (see Figure~\ref{fig: representation craft}).

\begin{figure*}[h]
    \centering
    \includegraphics[width=\linewidth]{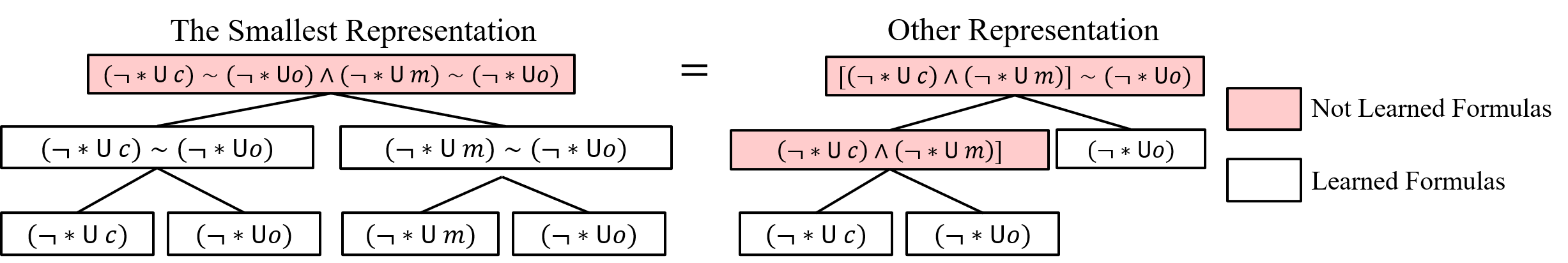}
    \caption{Representations of the task \textit{``delivering coffee and mail to office avoiding furniture''} in the \textsc{OfficeWorld} domain.}
    \label{fig: representation office}
\end{figure*}

\begin{figure*}[h]
    \centering
    \includegraphics[width=\linewidth]{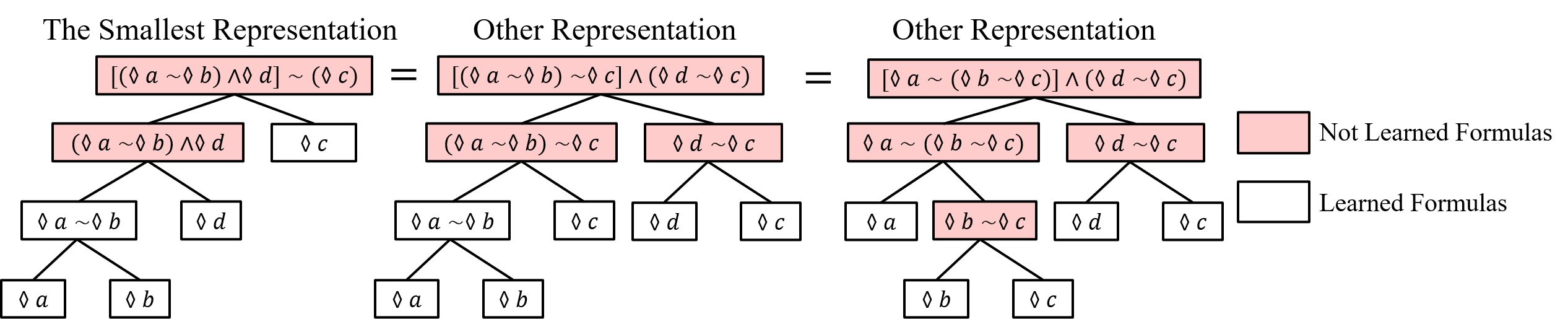}
    \caption{Representations of the task \textit{``making bed''} in the \textsc{MineCraft} domain}
    \label{fig: representation craft}
\end{figure*}

In Section 5.3, the sequences of tasks in each phase in the two domains are listed in Table \ref{tab:tasks in office} and \ref{tab:tasks in craft}, respectively.

\begin{table*}[h]
    \centering
    \begin{tabular}{|c||c|c|}
    \hline
        Phase & Task & The Smallest Representation  \\\hline
        \multirow{2}{*}{1} & \textit{deliver coffee to A avoiding furniture} & $(\neg *\until c)\sim (\neg *\until A)$\\\cline{2-3}
        & \textit{deliver mail to B avoiding furniture} & $(\neg *\until m)\sim (\neg *\until B)$ \\\hline
        \multirow{2}{*}{2} & \textit{deliver coffee to office avoiding furniture} & $(\neg *\until c)\sim (\neg *\until o)$\\\cline{2-3}
        & \textit{deliver mail to office avoiding furniture} & $(\neg *\until m)\sim (\neg *\until o)$ \\\hline
        \multirow{3}{*}{3} & \multirow{2}{*}{\textit{deliver coffee and mail to office avoiding furniture}} & $[(\neg *\until c)\sim (\neg *\until o)]$\\
        && $\land[(\neg *\until m)\sim (\neg *\until o)]$\\\cline{2-3}
        & \textit{go to B then A avoiding furniture} & $(\neg *\until B)\sim (\neg *\until A)$ \\\hline
    \end{tabular}
    \caption{Tasks in the \textsc{OfficeWorld} domain.}
    \label{tab:tasks in office}
\end{table*}

\begin{table*}[h]
    \centering
    \begin{tabular}{|c||c|c|}
    \hline
        Phase & Task & The Smallest Representation \\\hline
        \multirow{2}{*}{1} & \textit{make plank} & $(\Diamond a)\sim (\Diamond b)$\\ \cline{2-3}
        & \textit{make stick} & $(\Diamond a)\sim (\Diamond c)$ \\\hline
        \multirow{2}{*}{2} & \textit{make cloth} & $(\Diamond d)\sim (\Diamond e)$\\\cline{2-3}
        & \textit{make rope} & $(\Diamond d)\sim (\Diamond b)$ \\\hline
        \multirow{2}{*}{3} & \textit{make bridge} & $(\Diamond a \land \Diamond f)\sim (\Diamond e)$\\\cline{2-3}
        & \textit{make bed} & $((\Diamond a\sim\Diamond b) \land \Diamond d))\sim (\Diamond c)$ \\\hline
        \multirow{2}{*}{4} & \textit{make axe} & $((\Diamond a\sim\Diamond c) \land \Diamond f))\sim (\Diamond b)$\\\cline{2-3}
        & \textit{make shears} & $(\Diamond a\land \Diamond f) \sim (\Diamond c)$ \\\hline
        \multirow{2}{*}{5} & \textit{get gold} & $(\Diamond a\land \Diamond f) \sim (\Diamond e)\sim (\Diamond g)$\\\cline{2-3}
        & \textit{get gem} & $((\Diamond a\sim\Diamond c) \land \Diamond f))\sim (\Diamond b)\sim (\Diamond h)$ \\\hline
    \end{tabular}
    \caption{Tasks in the \textsc{MineCraft} domain.}
    \label{tab:tasks in craft}
\end{table*}

\section{Illustrations of the LSRM Processes in the Experiments}

Figure \ref{fig: extend rm 1} and \ref{fig: extend rm 2} give illustrations of LSRM processes in the \textsc{OfficeWorld} and \textsc{MineCraft} domain, respectively. The nodes (rectangles) encoded with SLTL formulas are states in RM, and the black solid arrows are the transitions among them. Formulas colored in white have already been learned before, while formulas colored in red have not been learned. The red dotted arrows represent the Q-functions transferred from the learned formulas to the corresponding target formulas. 
\begin{figure*}[h]
    \centering
    \includegraphics[width=\linewidth]{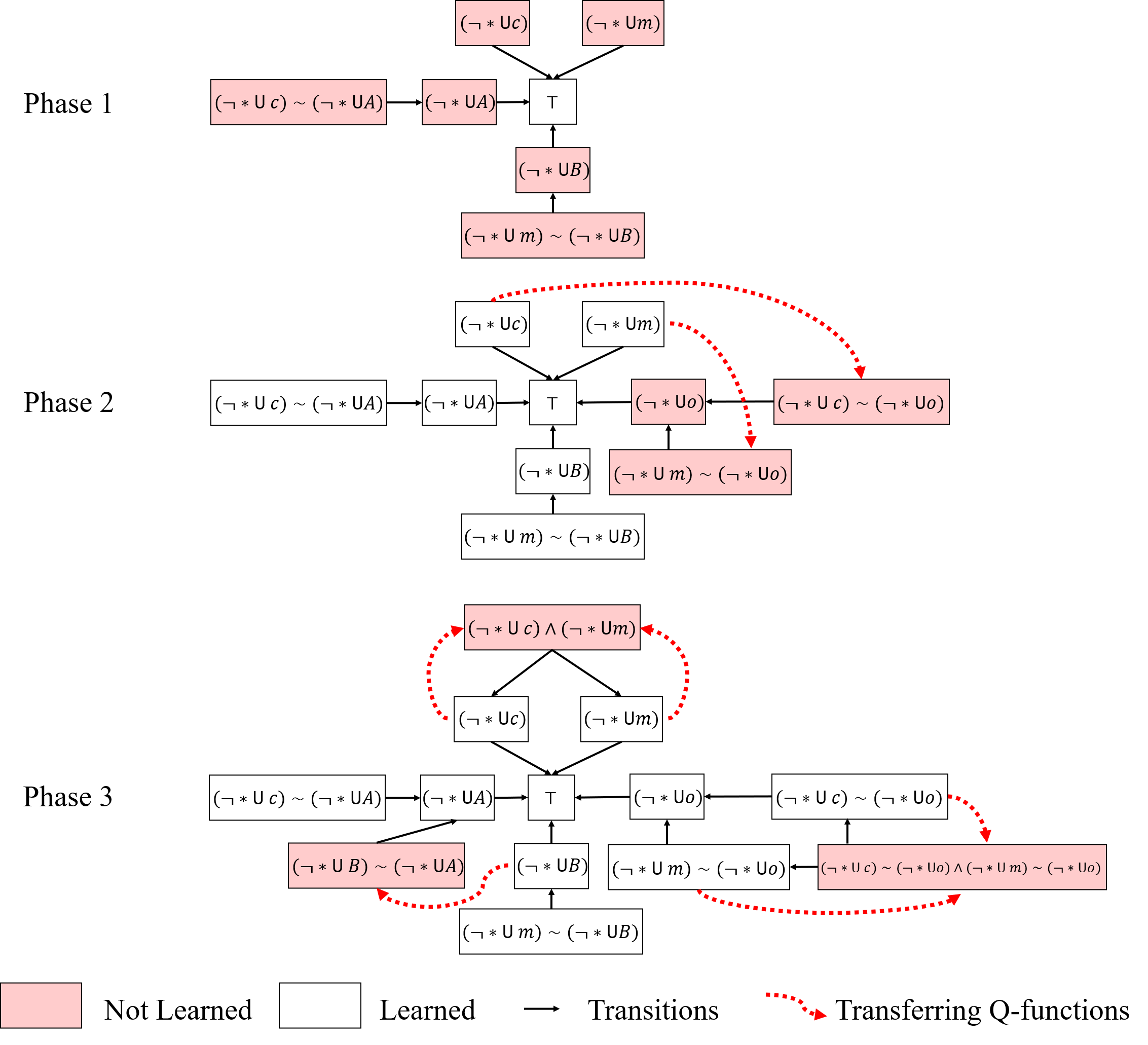}
    \caption{The LSRM process in the \textsc{OfficeWorld} domain. Taking Phase 2 as an example, formulas $(\neg *\until o),(\neg *\until c)\sim (\neg *\until o)$ and $(\neg *\until m)\sim (\neg *\until o)$ have not been learned. Q-functions of $(\neg *\until c)$ and $(\neg *\until m)$ are transferred to $(\neg *\until c)\sim (\neg *\until o)$ and $(\neg *\until m)\sim (\neg *\until o)$, respectively, while there are no learned formulas transferred to $(\neg *\until o)$.}
    \label{fig: extend rm 1}
\end{figure*}

\begin{figure*}[h]
    \centering
    \includegraphics[width=\linewidth]{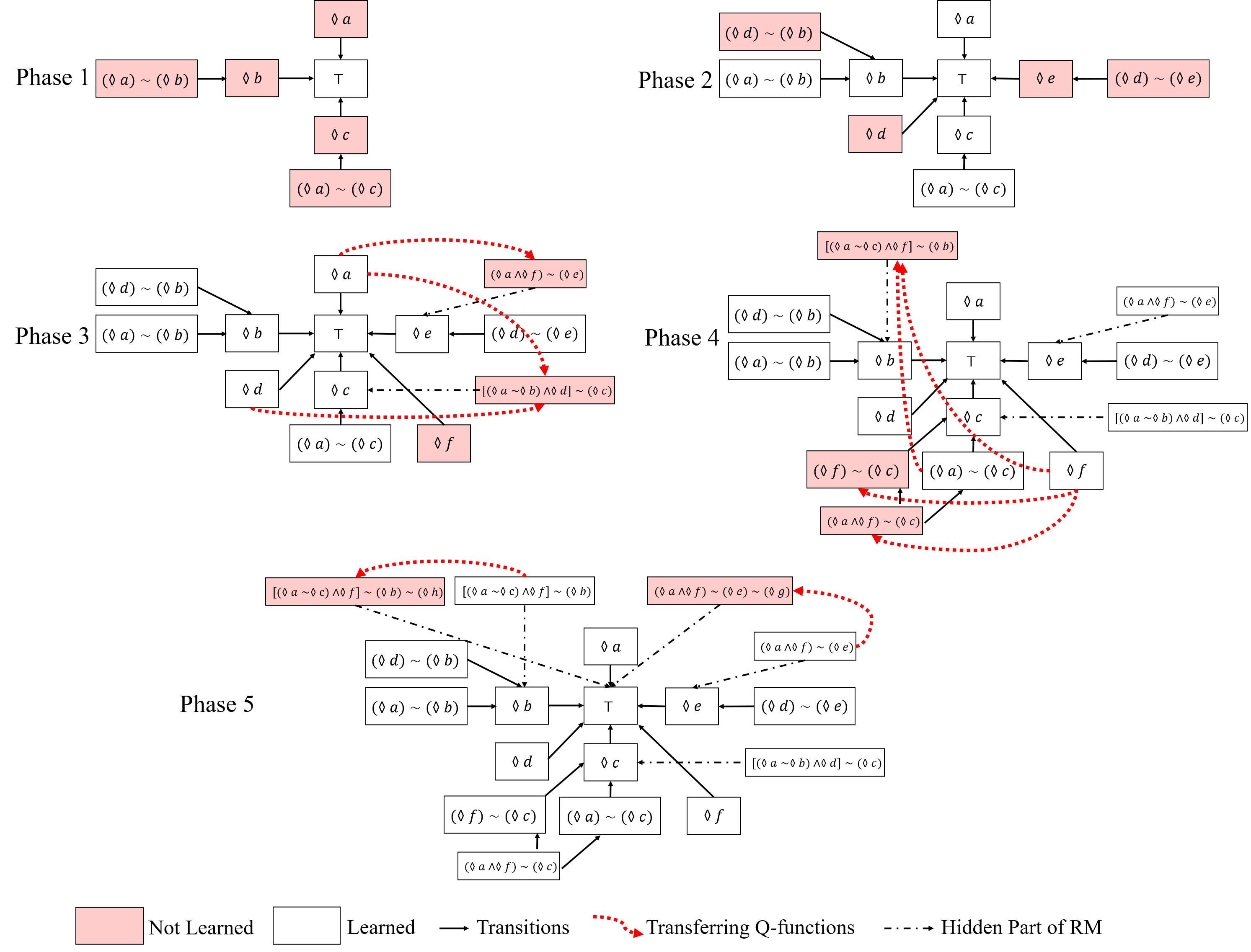}
    \caption{The LSRM process in the \textsc{MineCraft} domain. For the sake of succinctness, we use black dotted arrows to represent a hidden part of RM which includes multiple states and transitions in Phase 4-5.}
    \label{fig: extend rm 2}
\end{figure*}

\end{document}